%% file: main.tex
\pdfoutput=1

\documentclass[11pt]{article}

\usepackage{emnlp2021}

\usepackage{times}
\usepackage{latexsym}

\usepackage[T1]{fontenc}

\usepackage[utf8]{inputenc}

\usepackage{microtype}
\newcommand{\BB}{\text{BERT}_{\text{BASE}}}
\newcommand{\BS}{\text{BERT}_{\text{SMALL}}}

\input{header}

\input{math_commands}

\input{tables/attention_viz}
\input{tables/glue}

\input{tables/xtreme}
\input{tables/hyperparameters}

\input{tables/performance}

\input{tables/translate}
\usepackage{wrapfig}

\title{A Simple and Effective Positional Encoding for Transformers}

\author{Pu-Chin Chen\thanks{\quad The authors contribute equally to this paper. Corresponding author email: puchin@google.com} , 
Henry Tsai\footnotemark[1] , 
Srinadh Bhojanapalli\footnotemark[1] , \\
\textbf{Hyung Won Chung, Yin-Wen Chang, Chun-Sung Ferng} \\ \\
Google Research}

\begin{document}
\maketitle
\begin{abstract}
Transformer models are permutation equivariant. To supply the order and type information of the input tokens, position and segment embeddings are usually added to the input. Recent works proposed variations of positional encodings with relative position encodings achieving better performance. Our analysis shows that the gain actually comes from moving positional information to attention layer from the input. Motivated by this, we introduce \textbf{D}ecoupled pos\textbf{I}tional att\textbf{E}ntion for \textbf{T}ransformers (DIET), a simple yet effective mechanism to encode position and segment information into the Transformer models. The proposed method has faster training and inference time, while achieving competitive performance on GLUE, XTREME and WMT benchmarks. We further generalize our method to long-range transformers and show performance gain.
\end{abstract}

\input{1_intro}
\input{2_related}
\input{3_theorem}
\input{4_methods}
\input{5_experiments}

\section{Conclusion}
In this paper we theoretically and empirically examined the limitation of additive position embedding at input and showed that having per-head position embeddings results in better performance. We argued that the superior performance of some of the relative position encoding methods come from their per-head addition to attention matrix rather than the position information being relative vs absolute. Indeed we show that using absolute position encodings per-head results in better performance. Motivated by this we propose a simple per-head position and segment attention method that achieves the state-of-the-art performance on multiple NLP tasks and is more computationally efficient than existing approaches. 
\bibliography{anthology,custom}
\bibliographystyle{acl_natbib}

\appendix
\onecolumn
\input{6_appendix}

\end{document}

%% file: header.tex
\usepackage{caption}
\usepackage{subcaption}
\usepackage{graphicx} 
\usepackage{multirow}
\usepackage{xcolor}
\usepackage{mathtools}
\usepackage{relsize}
\usepackage{nicefrac}
\usepackage{xspace}
\usepackage{amsmath,amssymb,enumerate}
\usepackage{amsthm,cancel}
\usepackage{dsfont}
\usepackage[mathscr]{eucal}
\usepackage[inline]{enumitem}

\usepackage{booktabs}       
\usepackage{amsfonts}       
\usepackage{microtype}      
\usepackage{algorithm}
\usepackage{algorithmic}
\usepackage{url}
\usepackage{hyperref}

\usepackage[normalem]{ulem}
\usepackage{float}

\newtheorem{theorem}{Theorem}

\newtheorem*{lemma*}{Lemma}
\newtheorem*{theorem*}{Theorem}
\newtheorem*{assumption*}{Assumption}
\newtheorem*{corollary*}{Corollary}
\newtheorem*{remark*}{Remark}
\newtheorem*{definition*}{Definition}

\newcommand{\dietrel}{\textsc{Diet-Rel}}
\newcommand{\dietabs}{\textsc{Diet-Abs}}
\newcommand{\dietlin}{$\dietabs{}^{LIN}$}

\newcommand{\ours}{\textsc{Diet-Rel}}
\newcommand{\absscalar}{\textsc{Diet-Abs}}



%% file: math_commands.tex
\usepackage{amsmath,amsfonts,bm}

\def\1{\bm{1}}

\def\rmA{{\mathbf{A}}}

\def\rmE{{\mathbf{E}}}

\def\rmP{{\mathbf{P}}}

\def\rmR{{\mathbf{R}}}
\def\rmS{{\mathbf{S}}}

\def\rmW{{\mathbf{W}}}
\def\rmX{{\mathbf{X}}}

\def\mP{{\bm{P}}}

\def\mW{{\bm{W}}}

\DeclareMathAlphabet{\mathsfit}{\encodingdefault}{\sfdefault}{m}{sl}
\SetMathAlphabet{\mathsfit}{bold}{\encodingdefault}{\sfdefault}{bx}{n}

\newcommand{\R}{\mathbb{R}}

%% file: tables/attention_viz.tex
\newcommand{\insertDietAbsViz}{
\begin{figure}[th]
\small
\centering
\includegraphics[width=6.5cm]{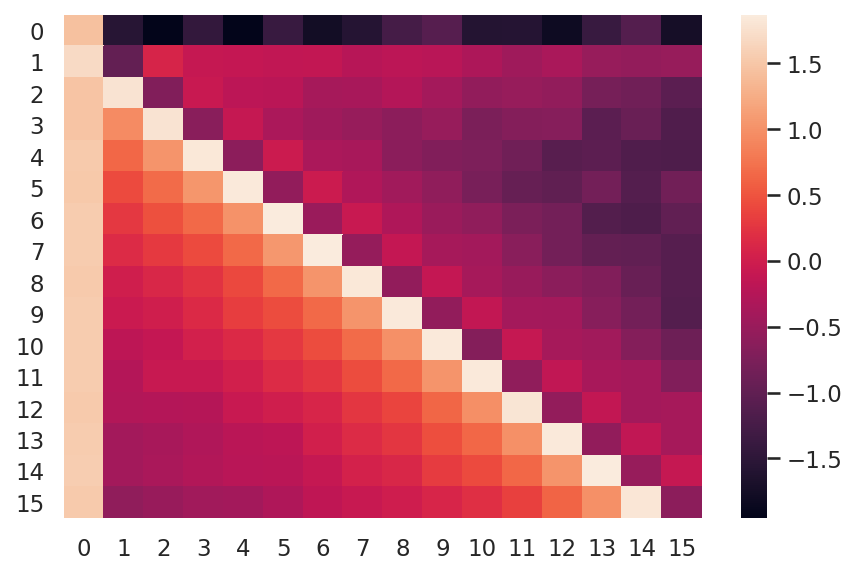}
\caption{Visualization of learned positional attention patterns of \dietabs{}. Note that in addition to capturing the the relative positional relations, the model also learn to attend to [CLS] at index 0, suggesting the dedicated [CLS] untying design in \citet{ke2020rethinking} is not necessary with \dietabs{}.}
\label{fig:attend_to_cls}
\vspace{-0.2in}
\end{figure}
}

%% file: tables/glue.tex
\newcommand{\insertGLUE}{
\begin{table*}[t]
\small
    \begin{center}
    \begin{tabular}{l | c | c | cccccc | c}
    \toprule
    \multirow{2}{*}{Model} &
    \multirow{2}{*}{Position} &
    \multirow{2}{*}{Segment} & \bf MNLI & \bf QQP & \bf QNLI & \bf SST2 & \bf CoLA &\bf STS-B & \bf \multirow{2}{*}{Avg} \\
    & & & 393k & 364k & 105k & 67k & 8.5k & 7k \\
   \midrule 
    \citet{devlin2018bert} & input & input & 85.8 / 85.9 & 91.1 & 89.9 &  93.2 & 58.7 & 89.0 & 84.8 \\
    \citet{shaw2018self} & per-head & input & 86.3 / 86.0 &  91.2 & 90.5 &  93.2 & 59.8 & 89.3 & 85.2 \\
    \citet{raffel2020exploring} & per-head & input & 86.4 / 86.2 &  91.2 & 90.1 & 93.0 & 59.6 &  90.1 &  85.2\\
    \citet{ke2020rethinking} & per-head & input & 86.1 / 86.2 &  91.2 &  90.3 & 93.1 & 59.6 &  89.6 & 85.2\\
    \dietrel  &  per-head & input & 86.0 / 86.1 & 91.0 & 89.8  & 92.8  & 59.6 & 89.0 & 84.9 \\
    \dietrel{} &  per-head & per-head & 86.3 / 86.3 & 91.0 & 90.5 & 92.9 & 60.3& 89.3 & 85.2 \\
    \dietabs{} ($d_p$=128) & per-head & per-head & 86.7 / 86.5  &  91.2 &  90.6 & 92.8  & 60.1 & 89.4 & 85.3 \\
    \dietabs{} ($d_p$=128, share) & per-head & per-head & 86.4 / 86.4 & 90.8	& 89.5	& 93.0	& 59.8	& 90.2 & 85.2 \\ 

    \midrule
    \citet{wang2020linformer} ($d_p$=32) & input & input &82.3 / 82.6&90.2&86.3&91.4&53.9&87.6 & 82.0 \\
    $\absscalar{}^{LIN}$ ($d_p$=32) & per-head & input &83.0 / 83.1&90.6&86.7&92.0&55.7&87.6&82.7\\
    \bottomrule
    \end{tabular}
    \end{center}
    \vspace{-0.3cm}
    \caption{
    GLUE: Results on the GLUE dev set of the finetuned models based on a pre-trained model with 12-layer $\BB$ architecture. We report the median of the maximum accuracy over all checkpoints among five runs. We notice that the shared \absscalar{} with rank 128 performs competitively  with existing relative positional embedding SoTA models without the inductive bias of the relative positions. The proposed method also improves performance in the low-rank long range transformer setting of ~\citep{wang2020linformer}, where relative positional embedding approaches are inefficient to use.}
    \label{tab:glue}
\end{table*}
}

%% file: tables/xtreme.tex
\newcommand{\insertXTREMEall}{
\begin{table*}[t]
\smaller
    \begin{center}
        \begin{tabular}{l | c| c | c | ccc| c}
            \toprule
             \multirow{3}{*}{Model} &
             \multirow{3}{*}{Position} &
             \multirow{3}{*}{Segment} &
             \multicolumn{1}{c|}{Classification} & \multicolumn{3}{c|}{Question Answering} & \multirow{3}{*}{\bf{Avg}} \\
             &&& \bf XNLI & \bf XQuAD & \bf MLQA & \bf TyDiQA & \\
             & & & 393k & \multicolumn{2}{c}{88k} & 3.7k & \\
             \midrule
            \citet{devlin2018bert} & input & input & 67.0 & 66.0 / 49.9 & 56.2 / 41.0 & 59.0 / 47.9 & 55.3\\
            \citet{shaw2018self} & per-head & input &  67.9 & 69.5 / 53.9 & 58.2 / 43.1 & 64.8 / 49.9 & 58.2 \\
            \citet{raffel2020exploring} & per-head & input &  68.5 & 69.9 / 53.5 & 59.5 / 44.3 & 63.8 / 50.6 & 58.6 \\
            \citet{ke2020rethinking} & per-head & input & 67.8 & 68.6	/ 52.0&	58.6 / 43.2	& 63.9 / 48.7 & 57.5 \\     
            \dietrel{} & per-head & input & 68.0 & 68.1 / 52.8 & 57.7 / 42.7 & 63.3 / 50.9 & 57.6\\
            \dietrel{} & per-head & per-head & 68.4 & 69.4 / 54.4 & 58.6 / 43.5 & 62.4 / 49.3 & 58.0 \\
           \dietabs{} ($d_p$=128, share) & per-head & per-head & 68.5 & 70.0 / 53.6 & 59.8 / 44.5 & 64.6 / 51.5 & 58.9 \\

            \midrule
            \citet{wang2020linformer} ($d_p$=256) & input & input & 63.6 & 59.1 / 43.7 & 48.9 / 34.0 & 50.5 / 37.9 & 48.2 \\
            \dietlin ($d_p$=256) & per-head & input & 64.4 & 61.6 / 46.0 & 52.2 / 37.0 & 53.6 / 40.9 & 50.8 \\
            \bottomrule
        \end{tabular}%
        \caption{XTREME: Fine-tune cross-lingual model on English training set (Cross-lingual Transfer). Performance is measured by \emph{accuracy} for classification, and \emph{f1 score / exact match} for question answering. In agreement with results in Table \ref{tab:glue} we see in this table that using per-head position encodings is strictly better than absolute position encodings at the input. With layer-wise sharing, \dietabs{} with rank 128 outperforms all SoTA models.}
    \label{tab:xtreme}
    \end{center}
  \vspace{-0.5cm}
\end{table*}
}

\newcommand{\insertXTREMEscalarsharing}{
\begin{table*}[t]
\small
    \begin{center}
        \begin{tabular}{l | cc | c | ccc | c}
            \toprule
             \multirow{2}{*}{Model} &
             \multirow{2}{*}{Sharing} &
             \multirow{2}{*}{Segment} &
             \multicolumn{1}{c|}{Classification} & \multicolumn{3}{c|}{Question Answering} &   \multirow{2}{*}{\bf{Avg}}\\
             & & & \bf XNLI & \bf XQuAD & \bf MLQA & \bf TyDiQA-GoldP \\
            \midrule
            \dietrel{} & - & input & 68.0 & 68.1 / 52.8 & 57.7 / 42.7 & 63.3 / 50.9 & 57.6\\
            \dietrel{} & head-wise & input & 67.7 & 66.2 / 51.0 & 56.0 / 41.1 & 60.1 / 45.9 & 55.4\\
            \dietrel{} & layer-wise & input & 68.0 & 68.6 / 53.3 & 58.1 / 43.1 & 61.3 / 48.2 & 57.2 \\

           \dietrel & - & per-head & 68.4 & 69.4 / 54.4 & 58.6 / 43.5 & 62.4 / 49.3 & 58.0 \\
            \dietrel{} & head-wise & per-head & 67.8 & 66.0 / 50.5 & 55.5 / 40.4 & 59.2 / 44.6 & 54.7 \\
            \dietrel{} & layer-wise  & per-head & 68.1 & 68.7 / 53.8 & 58.4 / 43.2 & 61.0 / 48.4 & 57.3 \\

            \midrule
            \dietabs{} ($d_p$=64) & - & input & 68.0 & 67.4 / 50.5 & 57.8 / 42.3 & 61.3 / 46.8 & 56.3 \\
            \dietabs{} ($d_p$=64) & - & per-head & 67.9 & 67.5 / 52.4 & 57.3 / 42.3 & 61.6 / 46.8 & 56.5 \\
            \dietabs{} ($d_p$=128) & - & per-head & 68.1 & 68.2 / 52.0 & 57.9 / 42.6 & 61.5 / 47.6 & 56.8 \\
            \dietabs{} ($d_p$=512) & - & per-head & 68.5 & 68.0 / 52.0 & 57.7 / 42.4 & 61.6 / 48.4 & 56.9  \\
           \dietabs{} ($d_p$=64) & layer-wise & input & 68.0 & 69.3 / 53.1 & 59.3 / 43.9 & 63.2 / 48.6 & 57.9 \\
           \dietabs{} ($d_p$=64) & layer-wise & per-head & 68.4 & 69.3 / 53.2 & 59.4 / 44.1 & 63.3 / 48.6 & 58.0 \\

           \dietabs{} ($d_p$=128) & layer-wise & per-head & 68.5 & 70.0 / 53.6 & 59.8 / 44.5 & 64.6 / 51.5 & 58.9 \\
           \dietabs{} ($d_p$=256) & layer-wise & per-head & 68.4 & 69.9 / 53.8 & 59.6 / 44.2 & 62.8 / 49.1 & 58.3 \\
           \dietabs{} ($d_p$=512) & layer-wise & per-head & 67.8 & 69.0 / 53.2 & 58.4 / 43.0 & 62.5 / 48.8 & 57.5 \\
            \bottomrule
        \end{tabular}%
        \caption{Ablation study on XTREME: We run decoupled positional attention ablation study to understand the effect of 1) sharing positional attention parameters across layers and heads 2) segment attention added at per-head 3) performance of relative and absolute 4) absolute positional attention rank $\it{d_p}$ from 64 to 512.}
    \label{tab:xtreme-scalar-sharing}
    \end{center}
\vspace{-0.1in}
\end{table*}
}

%% file: tables/hyperparameters.tex
\newcommand{\insertParameters}{
\begin{table*}[t]
\small
    \begin{center}
    \begin{tabular}{l | ccc | ccc}
      \toprule
      & \multicolumn{3}{c|}{\bf English} & \multicolumn{3}{c}{\bf Multilingual} \\
      & \multicolumn{1}{c}{Parameters} & \multicolumn{1}{c}{$+\Delta$} & \multicolumn{1}{c|}{GLUE} & \multicolumn{1}{c}{Parameters} & \multicolumn{1}{c}{$+\Delta$} & \multicolumn{1}{c}{XTREME} \\
      \midrule
      \citet{devlin2018bert} & 110.1M & - & 84.8 & 178.9M & - & 55.3 \\
      \citet{shaw2018self} & 112.9M & +2.5\% & 85.2 & 181.7M & +1.7\% &  58.2 \\
      \dietrel{} &  109.9M & +0.0\% & 85.2 & 178.7M & +0.0\%  & 58.0 \\
      \dietrel{} (share) & 109.7M & +0.0\% & 85.0 & 178.5M & +0.0\% & 57.3 \\
      \dietabs{} ($d_p$=128) & 128.6M & +16.8\% & 85.3 & 197.4M & +10.0\%  & 56.8 \\
      \dietabs{} ($d_p$=128, share) & 111.3M & +1.1\% & 85.2 & 180.1M & +0.6\% & 58.9 \\
      \bottomrule
    \end{tabular}
    \end{center}
  \vspace{-0.3cm}
    \caption{Model Parameters: We list the number of model parameters and performance for different position encoding approaches. We observe that sharing hurts the performance of \dietrel{} with negligible benefit in the number of parameters. On the contrary, the regularization effect of sharing makes \dietabs{} more stable with lesser parameters to achieve competitive performance.}
    \label{tab:parameters}
\end{table*}
}

%% file: tables/performance.tex
\newcommand{\insertPerformance}{
\begin{table*}[t]
\small
    \begin{center}
    \begin{tabular}{l|ccccc}
    & Mode & \citet{shaw2018self} & \citet{ke2020rethinking} & \absscalar{} & \ours\\
    \toprule
    $\BB$ & Training& +13\% & +1\% & +0\% & +0\%\\
    $\BB$ &Inference & +33\% & +19\% & +0\% & +0\%\\
    $\BS$ &Training & +24\% & +4\% & +0\% & +0\%\\
    $\BS$ &Inference& +65\% & +27\% & +1\% & +0\%\\
    \bottomrule
    \end{tabular}
    \end{center}
  \vspace{-0.3cm}
    \caption{Pre-training and inference time  of Transformers with different position encoding methods in comparison to the baseline BERT model on TPU v2. We observe that simplicity of the \ours{} and \absscalar{} result in substantial gains in both training and inference time. We notice even more speedup for the smaller $\BS$ model compared to $\BB$.}
    \label{tab:performance}
\vspace{-0.2in}
\end{table*}
}

%% file: tables/translate.tex
\newcommand{\insertTranslate}{
\begin{table}[t]
\small
    \begin{center}
    \begin{tabular}{l@{\hspace{4pt}}  |  @{\hspace{4pt}}c@{\hspace{4pt}}c@{\hspace{4pt}}c@{\hspace{4pt}}c}
    \toprule
    Model & \bf EN-DE & \bf DE-EN & \bf EN-CS & \bf CS-EN \\
   \midrule 
    \citet{vaswani2017attention} & 39.00 & 38.42 & 18.55 & 22.93 \\
    \citet{shaw2018self} &  40.10 & 38.90 & 18.74 & 23.89 \\
    \ours  & 39.47 & 38.49 & 18.68 & 23.93 \\
    \bottomrule
    \end{tabular}
    \end{center}
  \vspace{-0.3cm}
    \caption{Machine Translation: We report results comparing different position encoding methods for Transformers on machine translation tasks en-de, de-en, en-cs and cs-en from the Newstest 2018 dataset. We notice that all per-head position encoding schemes (all except the first row) do better than the absolute position embeddings added at the input. Further the proposed simple \ours{} approach is competitive with other position encoding approaches.
    }
    \label{tab:translate}
\end{table}
}

%% file: 1_intro.tex
\section{Introduction}
Transformers are sequence-to-sequence models that achieve state of the art performance in many Natural Language Processing (NLP) tasks, such as machine translation, language modeling and question answering \citep{vaswani2017attention, devlin2018bert, xlnet2019, liu2020deep}. Transformers have two major components: self-attention and a position-wise feed forward layer. Both are permutation equivariant and are not sensitive to the order of input tokens. To make these models position-aware, the position information of the input words is typically added as an additional embedding to the input token embeddings \citep{vaswani2017attention}. For example, input embedding ($\mW$) of a sentence is added to the position embeddings ($\mP$), resulting in input $\mW + \mP$ to the Transformer. These position embeddings only depend on the location the word appears. For multi-segment tasks, additional segment embeddings can be added just like the position embeddings \citep{devlin2018bert}.

There have been multiple works exploring different ways to include position information in Transformers \cite{shaw2018self, xlnet2019, raffel2020exploring}. Many of those note the advantages of using a relative position encoding scheme over absolute position encodings (see also Fig~\ref{fig:bar_plot}). However what causes this difference is not clear. \citet{yun2019transformers} have shown that Transformers with absolute position encodings are universal approximators of all sequence to sequence functions, proving that absolute position encodings can capture the position information. Hence what causes the superiority of relative position encodings? A systematic study and understanding of the benefits and drawbacks of different position encoding methods is missing. \citet{ke2020rethinking} hypothesised that the cross correlation between word and position embeddings while computing attention could be the cause of poor performance of absolute position encodings. However such cross terms are present in some of the relative position encoding methods \citep{shaw2018self, xlnet2019}, and these methods perform on par  or better than the other position encoding schemes (see \S \ref{sec:experiments}).

\begin{figure*}[t]
    \centering
    \subfloat[\centering English Transfer Learning on MultiNLI]{{\includegraphics[scale=0.28]{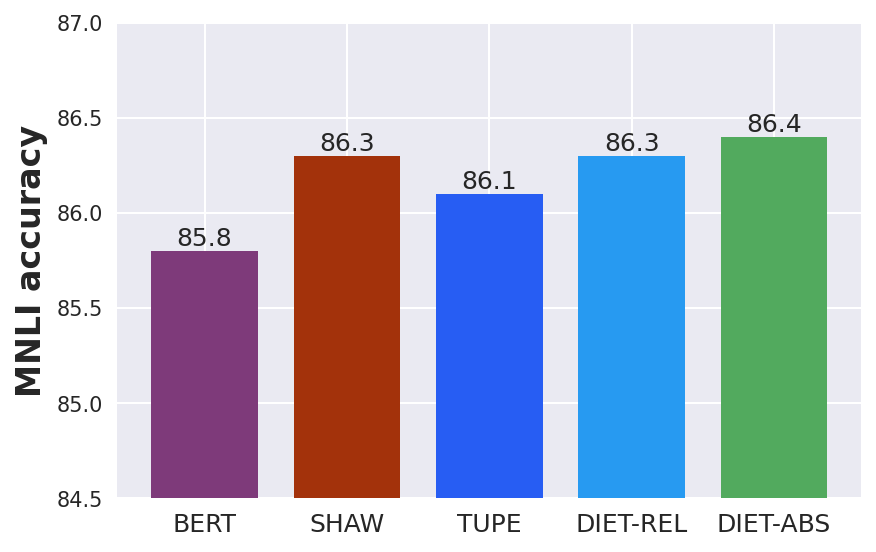} }}
    \qquad
    \subfloat[\centering Cross-lingual Transfer on XNLI]{{\includegraphics[scale=0.28]{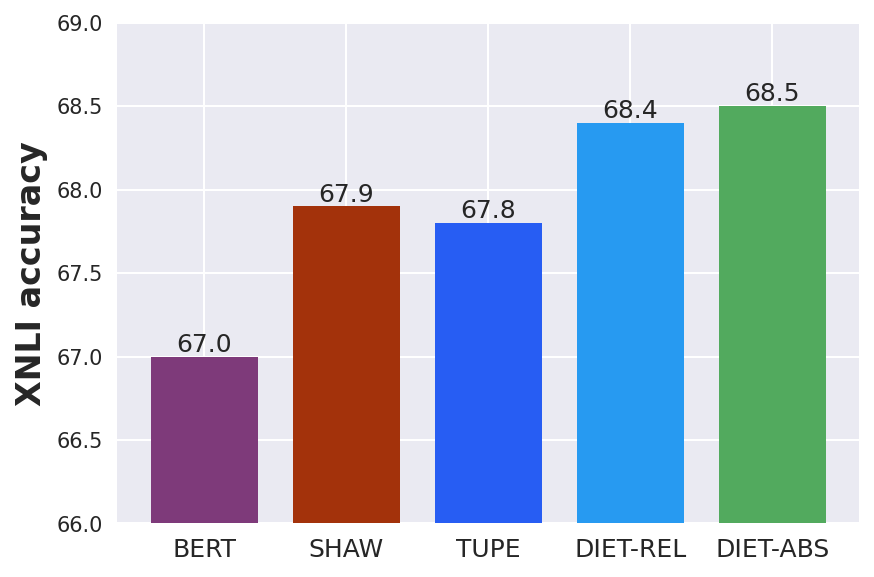} }}
    \qquad
    \subfloat[\centering Translation on CS-EN]{{\includegraphics[scale=0.28]{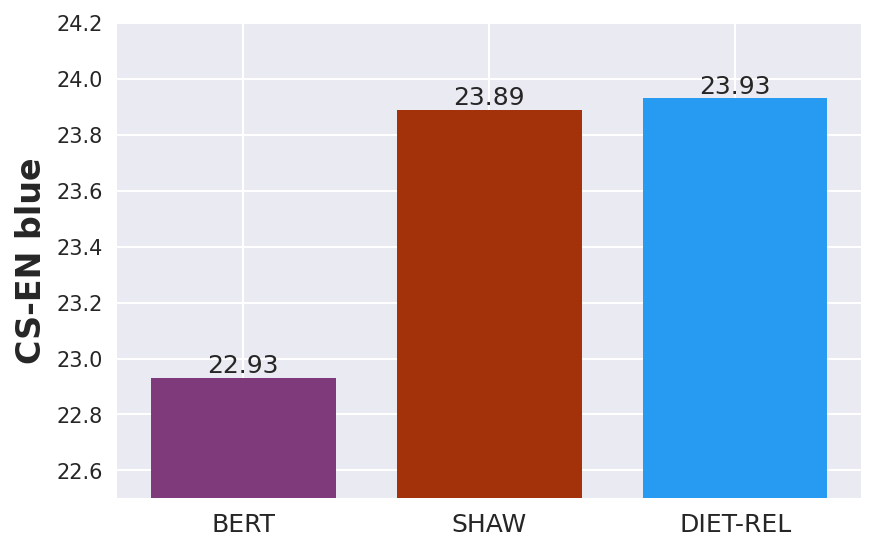} }}
    \caption{Performance effect of different positional encoding methods  for Transformers (see \S~\ref{sec:background}) on two Natural language Inference datasets from GLUE~\citep{wang2019glue}, XTREME~\citep{icml2020_4220} and one Neural Machine Translation dataset  WMT 18~\citep{bojar-etal-2018-findings}. Absolute positional encoding (\dietabs{}) can achieve better performance than the relative counterpart (\dietrel{}), showing the importance of designing the right position encoding method. }
    \label{fig:bar_plot}
\end{figure*}
In this paper we undertake a systematic study to understand different position encoding methods. We argue that absolute position embeddings mainly suffer from being added at the input. We show, with our experiments on classification, question answering and machine translation tasks, that absolute position encodings added to attention matrices with different parameters for each head improves significantly over absolute position encodings added to the input. This highlights that where the position information is included in the Transformer is important, providing an explanation for the gap in performance between absolute and relative position encodings. We also compare different position encodings and the effect of sharing position encodings across different heads and layers of a Transformer. Based on these observations we propose decoupled positional attention and a new segment encoding approach (for tasks with multiple segments), and empirically show its superiority.

We summarize our contributions in this paper below.
\begin{itemize}
\setlength\itemsep{0em}
    \item We theoretically and empirically analyze the limitation of the absolute position embeddings added to the input. For both absolute and relative information, we show that encoding position to attention matrix per-head results in superior performance.
    \item We propose a simple and efficient way to encode position and segment information. The proposed encoding matches the SoTA methods on multiple standard NLP tasks while having a simpler model with lower training/inference costs.
    \item Our proposed method can be easily applied to long sequence models (\dietlin{}) and improve all metrics compared with Linformer \citep{wang2020linformer}.
    \item We present ablation studies comparing different position encoding methods and ways of sharing position encoding parameters across heads and layers in Transformer.  
\end{itemize}

%% file: 2_related.tex
\section{Position Encoding for Transformers} \label{sec:background}

In this section, we briefly review the Transformer models \citep{vaswani2017attention} and  discuss previous improvement of position encoding and analyze the limitation of the additive position embedding proposed in the initial and widely-adopted Transformer model. 

\subsection{Transformer}
A Transformer block consists of two types of layers: 1) Self-attention layer and 2) Feed forward layers. 

\paragraph{Self-Attention Module}
Given input sequence length $n$, hidden size $d$, multi-head query-key down-projection size $d_h$, we define hidden layer input to this attention head as
$\rmX \in \mathbb{R}^{n \times d}$, the query projection matrix as $\rmW_Q^i \in \mathbb{R}^{d \times d_h}$, the key projection matrix as $\rmW_K^i  \in \mathbb{R}^{d \times d_h}$ and the value projection matrix as $\rmW_V^i  \in \mathbb{R}^{d \times d_h}$, $i \in [h]$, for $h$ heads. Usually, $d_h < d$ as we do multi-head attention with a smaller representation per head ($d_h = d/h$). With that we can write dot-product attention score: \begin{gather*}
  \rmA^i =  ( \rmX \rmW_Q^i)( \rmX \rmW_K^i)^\top
\end{gather*} This attention score is used to compute the output for each head, after scaling and per row normalization using softmax:
\begin{gather*}
\label{eq:attention_output_matrix}
  \text{head}^i =  \text{Softmax}(\rmA^i/\sqrt{d}) \cdot (\rmX\rmW_V^i)
\end{gather*}

Output of all attention heads in a layer are concatenated and passed to the next feed-forward layer applied token-wise.

\subsection{Position Aware Self Attention}
Many NLP tasks, such as machine translation, language modeling, are sensitive to the ordering of input words. Since Transformers are permutation equivariant, we usually additionally include the position information in the input. Below we discuss some of the popular position encoding methods.

\subsubsection{Absolute Position Encodings}
Absolute position encodings are computed in the input layer and are summed with the input token embeddings. \citet{vaswani2017attention} proposed this for Transformers and it has been a popular choice in the followup works \citep{radford2018gpt, devlin2018bert}. There are two common variations of the absolute position encodings - fixed and learned.

\subsubsection{Relative Position Encodings}
One drawback of absolute position encoding is that it requires fixed length of input sequence and does not directly capture relative positions to each word. To solve these problems several relative positions schemes have been proposed.

\citet{shaw2018self} proposed using relative position encoding instead of absolute position encoding, and add position embeddings to the key and optionally value projections instead of the input. They show that this new way of encoding position information leads to better performance on machine translation tasks. \citet{xlnet2019} simplified this by removing the position embeddings in value projections and showed better performance on the language modeling tasks. Both these approaches use a vector representation to encode position information.

\citet{raffel2020exploring} use scalars to encode relative position between query and key indices and add directly to the attention scores matrix. They further use logarithmic binning of position information into a fixed number of buckets. All these relative position methods further share the position encoding parameters across layers.

Recently \citet{ke2020rethinking} hypothesised that the cross correlation between position and token embeddings can result in weaker performance of additive absolute position embeddings and instead proposed to add both absolute and relative positional information based attention directly in each head. However such cross terms are present in the method proposed by \citet{shaw2018self}, which does competitively with other approaches. We instead hypothesise that position encodings at input limit the rank of the position attention matrix leading to its poor performance.

%% file: 3_theorem.tex
\subsection{Limitations of the Input Additive Position Embedding}
\begin{figure}
    \centering
    \includegraphics[width=7.5cm]{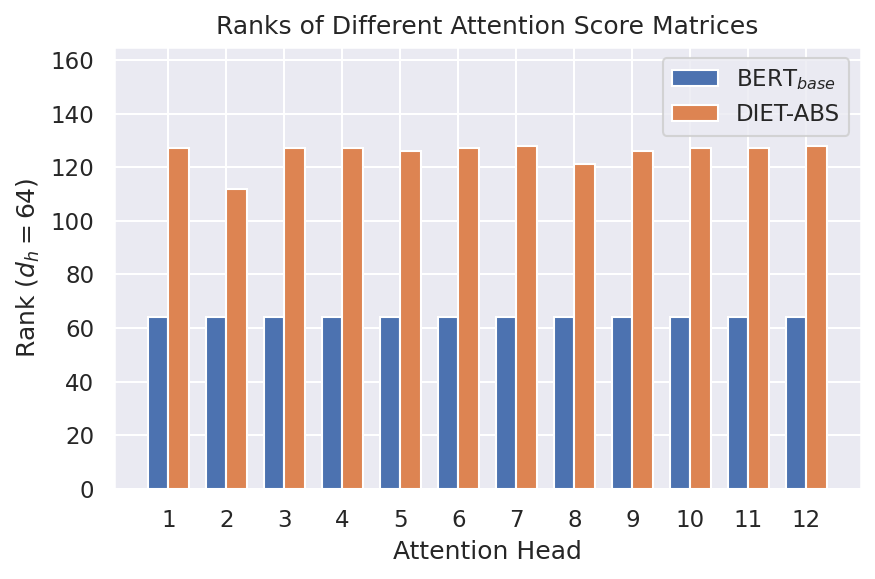}
    \caption{{Rank of attention matrices:} We present a comparison of the rank of  the attention score matrices of a $\BB$ model with absolute position embeddings at input v.s. absolute position embeddings per-head (\absscalar{}~\eqref{eq:absscalar}). With additive positional embedding at input, the attention matrices have much lower rank, limiting the representative power. This is alleviated by \absscalar{}.}
    \label{fig:bert_base_attention}
\end{figure}
In this section we discuss some limitations of the de facto way of adding absolute position encodings to the input token embeddings. 

We first compare the representation power in terms of the rank of attention matrices achievable with different position encodings.

\begin{theorem} \label{thm:rank}
Let $\rmP \in \R^{n \times d}$ be the input position embedding and $\hat{\rmP} \in \R^{n \times d_p}$ be the layer-wise position embeddings. Let $\rmW_Q, \rmW_K \in \R^{d \times d_h}$ be the query and key projection matrices with head projection size $d_h$, and $d_h < d_p, d$ and $n \geq d_h + d_p$. Let $\rmA_a = (\rmX+\rmP) \rmW_Q \rmW_K^\top (\rmX+\rmP)^\top$ and $\rmA_r = \rmX \rmW_Q \rmW_K^\top \rmX^\top + \hat{\rmP} \hat{\rmP}^\top $ be the attention matrices computed using input and layer-wise position embeddings respectively. Then for any $\rmX, \rmP, \rmW_Q, \rmW_K$ $$rank(\rmA_a) \leq d_h.$$ There exists a choice of $\rmX, \hat{\rmP}, \rmW_Q, \rmW_K$ such that $$rank(\rmA_r) = d_p+d_h > d_h.$$

\end{theorem}
\paragraph{Remarks.} This theorem shows us that the rank of attention matrices is constrained with the absolute position encodings at the input and using per-head position encodings by adding position information to attention matrix directly results in allowing for higher rank attention. See \S~\ref{appendix:proofs} for the proof.

Adding the position encodings directly to the input further places a constraint on training dynamics by forcing gradients to be same for both the input token and position embeddings (see \S~\ref{appendix:proofs}).  Relative position encodings discussed earlier, while addressing some of these concerns, suffer from slower training/inference times (see Table~\ref{tab:performance}) with complex implementations (\citet{shaw2018self, ke2020rethinking}). In the next section, we present simple position encoding methods that avoid these limitations.

%% file: 4_methods.tex
\section{Proposed Position and Segment Encodings} \label{sec:method}
In the previous section, we learned about the limitations of input additive positional embeddings and existing works. Based on these observations, we propose two minimal/efficient ways to incorporate (absolute/relative) positional encodings along with a novel absolute segment encoding approach. By decoupling position and segment from token embeddings we match the SoTA performance while improving  training/inference time (see \S \ref{sec:cost}).

\begin{figure*}[t]
    \centering
    \subfloat[\centering 
    \dietabs{}]{{\includegraphics[scale=0.35]{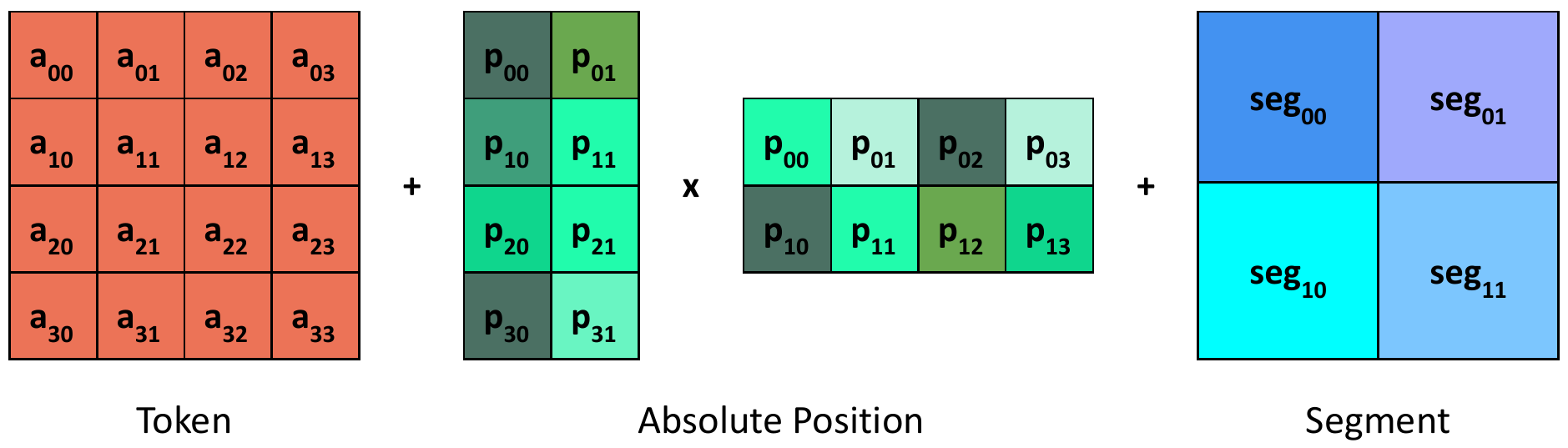} }}
    \qquad \qquad
    \subfloat[\centering 
    \dietrel{}]{{\includegraphics[scale=0.35]{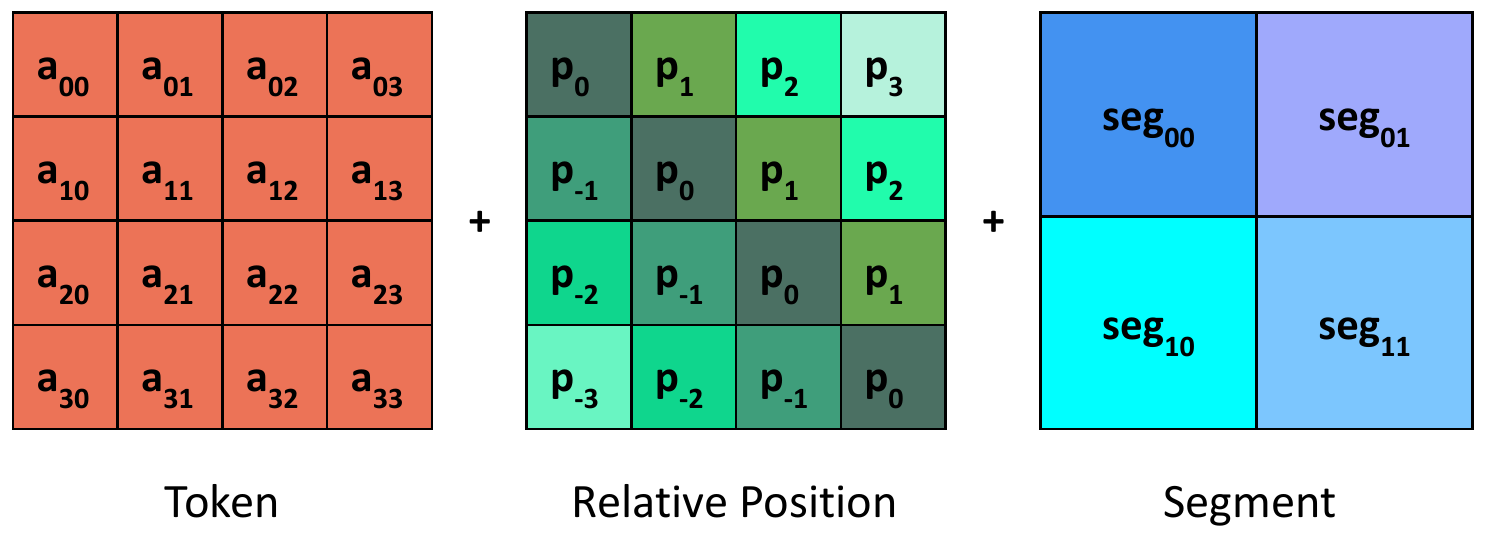} }}
    \caption{Proposed efficient approach to include position and segment encoding by adding them directly to the token attention matrix per-head. Left figure shows how we encode absolute positional attention. Right figure represents relative positional attention.}
    \label{fig:pe_attention}
\end{figure*}

\subsection{Decoupled Absolute Positional Attention}

We propose the following simple absolute position encoding method that adds position information to the token attention matrix directly in each attention head. We further also add segment information to the token attention instead of the input embeddings. This way we can set the rank of position encodings independently resulting in higher rank attention matrix, addressing the limitations discussed earlier.  
\paragraph{\absscalar}
\begin{align} 
\begin{split} 
 \rmA_{i,j}^{\text{ABS}} & = (\rmX_{i:} \rmW_Q) (\rmX_{j:} \rmW_K )^\top/ \sqrt{d} \\ 
 & + (\rmP_Q\rmP_K^\top)_{i, j} + E_S(S(i), S(j)),  \label{eq:absscalar}
\end{split}
\end{align}
where $\rmP_Q, \rmP_K \in \R^{n \times d_p}$ are low-rank position embedding matrices and $E_S$ is the absolute segment attention to model interactions between segments defined as 
\begin{equation} \begin{gathered} \label{eq:seg-def}
    E_S(S(i), S(j)) = \rmS_{\hat{i},\hat{j}}  \\
\text{ where }  {S}(i) = \hat{i} \text{ if index } i \text{ is in segment } \hat{i}.
\end{gathered} \end{equation}

Please note that we use the following notation in the above equation. $\rmA_{i, j}$ denotes the $(i, j)$ entry of matrix $\rmA$.  $\rmX_{i:}$  and $\rmX_{:j}$ denote the $i$th row and $j$th column of $\rmX$ respectively. We will follow this notation in the remainder of the paper. 

By default, we set $d_p$ same as $d_h$. This already results in potentially a rank $d_p+d_h$ attention matrix as shown in Theorem~\ref{thm:rank}. To illustrate this, we compare the rank of the attention matrices  in the first layer of a baseline BERT model and a \absscalar{} model for a sampled batch in Figure \ref{fig:bert_base_attention}. The figure shows that attention matrices of \absscalar{} have higher ranks than the baseline BERT. Our detailed experiment results in \S~\ref{sec:experiments} also show that \absscalar{} performs noticeably better. This confirms our earlier observation in Theorem~\ref{thm:rank} that additive position embeddings at input can constrain the model and adding the position embeddings per-head removes this constraint and results in better performance.

With the decoupled positional embedding, we can increase $d_p$ to any width $k$ to break the low-rank bottleneck shown in Theorem~\ref{thm:rank}. We call such model \absscalar{}-Rank-$k$.  We also address the efficiency issue introduced by one additional matrix multiplication ($\rmP_Q\rmP_K^\top$). As the positional embeddings are independent of the input, we only need to compute the matrix multiplication once for each training batch, and  we can cache the computed matrix before running inference. As a result, we observe neglectable training and inference cost increase in this model variant. 

\subsection{Decoupled Relative Positional Attention}

To incorporate relative position inductive bias, we consider a simplified version of the position encoding proposed in T5 \cite{raffel2020exploring} without log-binning and per-layer parameter sharing. We further also incorporate our per-head segment encoding as in \absscalar{}. The model can be written as:

\paragraph{\dietrel} 
\begin{align} 
\begin{split} 
\label{eq:minpe}
    \rmA_{i,j}^{\textrm{REL}} & = (\rmX_{i:} \rmW_Q) (\rmX_{j:} \rmW_K)^\top / \sqrt{d} \\
    & + \rmR_{i-j} + E_S(S(i), S(j)). 
\end{split}
\end{align}
We show an example of this model with two segments in Figure~\ref{fig:pe_attention}.

\subsection{Training and Inference Costs}
\label{sec:cost}
\insertPerformance

We next show the proposed models introduce little computational overhead compared to the baseline model, making our model more practical than alternatives. We consider two different models - $\BB$ model and a smaller model, $\BS$, that has hidden size 512, 4 layers and 8 attention heads.

In Table \ref{tab:performance} we compare the training and inference costs of position encoding methods of \citet{shaw2018self}, \citet{ke2020rethinking}, \dietabs{} and \dietrel{}. We notice that the simplicity of the proposed methods indeed translates to savings in both training and inference times compared to other position encoding approaches. The savings in step times are even more significant for smaller models ($\BS$) and during inference.

Note that the discrepancy between training and inference speed is likely because gradient updates dominate the cost at training time \cite{lan2020albert}. At inference time, we only measure the time of a forward pass which corresponds to costs of using such models in real systems.

\subsection{Application to Long-range Transformers}
Another advantage of our propose approaches is they easily extend to long range Transformer models. For long sequence inputs, Transformers suffer from quadratic dependence of computational complexity with respect to the sequence length. A class of methods reduce this complexity by using a low rank projection of the input sequence for attention computation \citep{wang2020linformer, choromanski2021rethinking,dai2020funneltransformer}. However, such methods use the default input position encodings, and there has not been much work in incorporating position information per-head without introducing the quadratic computation complexity on the input sequence length. We illustrate the applicability of our methods to such settings by applying \absscalar{} to Linformer \cite{wang2020linformer}, which projects the attention key and value matrices to a lower dimension $k$ during attention computation.

\paragraph{$\absscalar{}^{LIN}$} The proposed method can be written as:
\begin{align} 
\begin{split} 
 \rmA_{i,j}^{\text{LIN}} & = (\rmX_{i:} \rmW_Q) ((\rmE \rmX)_{j:} \rmW_K )^\top/ \sqrt{d} \\ 
 & + (\rmP_Q \rmP_K^\top)_{i, j}, 
 \label{eq:lin-scalar}
\end{split}
\end{align}
where $\rmE \in \R^{k \times n}\text{, }\rmP_Q \in \R^{n \times d}\text{, }\rmP_K \in \R^{k \times d}$.

%% file: 5_experiments.tex
\insertGLUE
\insertXTREMEall

\insertTranslate
\section{Experiments} \label{sec:experiments}

In this section, we present our experimental results comparing different position and segment encoding approaches discussed in earlier sections. We conduct experiments in three different settings to cover a wide range of use cases. First, we examine the results of a popular transfer learning approach from masked-LM pretraining to the end tasks in GLUE \citep{devlin2018bert}. Second, we study zero-shot cross-lingual transferability of the multilingual pretrained models \citep{icml2020_4220} to classification and question answering tasks in the XTREME benchmark \citep{icml2020_4220}. Lastly, we consider training Transformer models from scratch for machine translation. 

We compare the following positional encoding approaches - absolute positional embedding \citep{devlin2018bert}, relative positional embedding \citep{shaw2018self}, combined absolute and relative positional encoding \citep{ke2020rethinking}, relative scalar approach~\citep{raffel2020exploring}, our proposed \dietabs{} and \dietrel{} per-head positional encoding approaches. We denote the methods that add position/segment information directly to input token embeddings with \textit{input}, and methods that add position/segment information directly in attention layer with \textit{per-head}.  For complete experimental setup, see Appendix~\ref{appendix:experiments}.

\subsection{English Transfer Learning Results}
\paragraph{Datasets and Model} For pre-training, we use English Wikipedia and Books datasets \cite{devlin2018bert}. For Finetuning tasks we use the datasets from the GLUE benchmark \citep{wang2019glue}. We apply sub-word tokenization on raw text data using WordPiece \citep{wu2016googles} with a 30,000 token vocabulary.

\paragraph{Results}
We examine how different ways of encoding position and segment affect the transfer learning ability of the pre-trained English BERT models by fine-tuning on the GLUE benchmark \cite{wang2019glue}, and present the results in Table~\ref{tab:glue}. We first notice that all the approaches that encode position features explicitly at per-head level perform better than the baseline additive position encodings at the input \citep{devlin2018bert}. All models incorporating relative positions \cite{shaw2018self,raffel2020exploring,ke2020rethinking}, despite their modeling differences, have very similar average score. We show further gains (84.9 to 85.2 for \ours{}) by moving segment features to per-head. 

Interestingly we notice that the proposed absolute position encoding method \dietabs{}, with layer-wise sharing, is on par with all previous SoTA relative positional encodings. This shows that even absolute position encodings can perform better when included per-head instead at the input. We present a detailed ablation study varying the rank and sharing methods of absolute positional attention (\absscalar{}) in Table~\ref{tab:glue_ablation_model} and Tables~\ref{tab:glue_ablation_sharing} in Appendix~\ref{appendix:visualization}.

For long range input, we consider Linformer \cite{wang2020linformer} with a projection dimension of 32. Due to down-projection, we see non-trivial performance drop, when compared to a Transformer. Even for this setting we see that our absolute positional attention \dietabs{} can be used to improve the model's performance. 

\subsection{Cross-lingual Model Results}
\paragraph{Datasets and Model} 
For our multilingual experiments, we pre-train the models on Wikipedia corpus in 100 languages similar to \cite{lample2019crosslingual} for 125K steps with a sequence length of 512, and then fine-tune on downstream XTREME tasks \cite{icml2020_4220}. We use language-independent tokenizer, Sentence Piece \citep{kudo2018sentencepiece} model, with 120,000 token vocabulary to encode input text.


\paragraph{Classification} We conduct 5 trials of fine-tuning for each model on the MultiNLI \cite{williams2018broadcoverage} training data, then perform zero-shot predictions on XNLI \cite{conneau2018xnli}, choosing median accuracy to report. 

\paragraph{Question Answering} We conduct 5 trials of fine-tuning for each model on SQuAD V1.1 dataset, following by zero-shot predictions on XQuAD (11 languages), MLQA (7 languages) and TyDiQA-GoldP (9 languages), choosing median F1 / EM scores to report. 

\paragraph{Results}
We present our results on the classification and question answering finetuning tasks  in XTREME for different position and segment encoding methods in Table~\ref{tab:xtreme}. Again all per-head position encoding methods outperform input additive position encodings. Interestingly, our simple \dietabs{} turns out to be the best model, better than other models using relative position features.  Layer-wise sharing and per-head segment attention allows \dietabs{} to outperform \dietrel{}. We present a detailed ablation study in Table~\ref{tab:xtreme-scalar-sharing} to understand effect of decoupled positional attention variants. Finally, we notice similar advantages in using \dietabs{} with the Linformer \citep{wang2020linformer} model in the long range setting.

\insertXTREMEscalarsharing
\insertParameters

\subsection{Translation Results}

\paragraph{Datasets and Model} 
For the machine translation task we consider two language pairs (both directions) for training - WMT 2018 English-to-German (en-de), German-to-English (de-en), English-to-Czech (en-cs) and Czech-to-English (cs-en) \citep{bojar-etal-2018-findings}.  We test the corresponding models on Newstest 2018 datasets respectively and report the BLEU score output by SacreBLEU \citep{Post2018ACF} with default setting. Our setup follows \citet{vaswani2017attention} closely and use their Tensor2Tensor framework \citep{vaswani2018tensor2tensor}. Following \citet{vaswani2017attention} we use a 6 layer Transformer with encoder-decoder architecture. For more details of our experimental setup please see Appendix~\ref{appendix:experiments}

\paragraph{Results}
We report the BLEU scores of the models in Table \ref{tab:translate}. We observe that moving positional information from input to per-head attention layer improves BLEU scores. Different variations of per-head positional attention do not make much difference with \ours{} being competitive with \citet{shaw2018self}.

\subsection{Ablation Study}
In this section, we share our findings of key factors that affect performance of decoupled positional attention.
\paragraph{Sharing the Positional Encoding}
Previous works \cite{raffel2020exploring, ke2020rethinking, shaw2018self} used different sharing methods for the positional encodings to reduce the model parameters. We present a detailed study on different forms of sharing positional encodings and its effect on performance. In particular, we compare the following variations in sharing the position encoding parameters across different heads and the layers in the Transformer.
\begin{itemize}
    \item head-wise - Same parameters are used for all heads in a layer, with different layers using different parameters \cite{shaw2018self, ke2020rethinking}.
    \item layer-wise - Sharing of position encoding parameters across layers with different parameters for each head \citep{raffel2020exploring}. 
    \item none - Every layer and head  uses different position encoding parameters.
    \end{itemize}
We present results comparing different sharing methods in Table~\ref{tab:xtreme-scalar-sharing} for XTREME tasks. We make the following observations 1) head-wise sharing is consistently worse than layer-wise, 2) sharing hurts the performance of \dietrel{} whereas it improves the performance of \dietabs{}. We summarize the key settings along with the number of model parameters in Table \ref{tab:parameters}. For \dietrel{}, sharing brings little effect on saving parameters, and hurts the performance. Hence, we recommend no sharing for relative positional encodings (\dietrel{}). On the other hand, it is necessary to share parameters for \dietabs{} in order to keep the number of parameters low. Interestingly, sharing has regularization effect on \dietabs{}, making the model perform better. We choose layer-wise sharing over head-wise sharing for its better performance.

\paragraph{Segment Encoding}
Our novel segment encoding design further improves the model performance showed in Table~\ref{tab:xtreme-scalar-sharing}. Both relative and absolute decoupled positional attention models benefit from moving the segment encoding from input to per-head: \dietrel{} (+0.4\%), layer-wise shared \dietrel{} (+0.1\%), \dietabs{} (+0.2\%), layer-wise shared \dietabs{} (+0.1\%). See Appendix~\ref{appendix:glue} for the results of GLUE benchmark and Appendix~\ref{appendix:visualization} for segment attention visualization.

\paragraph{Rank of Absolute Positional Attention}
The design of \dietabs{} allows to learn higher rank attention matrices as shown in Theorem~\ref{thm:rank}. To understand the effect of absolute positional attention rank ($\it{d_p}$) in practice, we conduct experiments varying the rank from $d_p=64$ to $d_p=512$. We present the results in Table \ref{tab:xtreme-scalar-sharing}. We notice that the performance improves as we increase the rank from 64 to 128. However there is a performance saturation in further increasing it to 512. We present a visualization of the rank of the positional attention matrix  in Appendix~\ref{appendix:proofs}.
\insertDietAbsViz

\subsection{Positional Attention Pattern Visualization} \label{sec:visualization}
We next visualize the learned positional attention patterns of \dietabs{} in Figure~\ref{fig:attend_to_cls}. We first note that \dietabs{} has learned to capture the relative positional relations between inputs. Also note that, for the the index zero (the [CLS] token), decoupled absolute positional attention usually learns a special pattern. This pattern cannot be solely modeled by existing relative positional embedding methods, and some existing works~\citep{ke2020rethinking} handled this case specifically by introducing new parameters.  This shows the benefit of \dietabs{} in not requiring any carefully designed inductive biases as in existing approaches( \citet{shaw2018self, raffel2020exploring}), which may not generalize across tasks.

%% file: 6_appendix.tex
\section{Experimental setup} \label{appendix:experiments}

In this section we present more details of our experimental setup.

\paragraph{Pre-training} We pre-train the models using a masked LM task \citep{devlin2018bert} and do not use the Next Sentence Prediction (NSP) loss as suggested in RoBERTa \citep{roberta2019}. Each input is constructed with full sentences from documents, and packed up to the maximum sequence length. We use the same architecture as $\BB$ \citep{devlin2018bert} ($L$ = 12, $H$ = 768, $A$ = 12) for our experiments. 

\paragraph{Fine-tuning} Some downstream tasks have different groups of full sentences provided at inputs. For those tasks (e.g. MNLI, CoLA, XNLI, SQuAQ), we fine-tune models with supplemental segment encoding discussed in Section \S \ref{sec:method}. We leave models for other tasks unchanged as their pre-training correspondences.

\paragraph{Hyper-parameters} Hyper-parameters we use are presented in Table~\ref{tab:hyperparameters}.
\begin{table}[H]
\small
    \begin{center}
    \begin{tabular}{l | rr | rr}
      \toprule
      & \multicolumn{2}{c|}{\bf English} & \multicolumn{2}{c}{\bf Multilingual} \\
      & \multicolumn{1}{r}{Pretrain} & \multicolumn{1}{r|}{Finetune} & \multicolumn{1}{r}{Pretrain} & \multicolumn{1}{r}{Finetune} \\
      \midrule
      Max Steps & 500K &5 or 10 epochs & 125K & 3 epochs \\
      Learning Rate &  0.0018 & \{1e-5, 2e-5, 3e-5, 4e-5\} & 0.0018 & \{1e-5, 2e-5, 3e-5, 4e-5\} \\
      Warmup Proportion &  0.025 & 0.1 & 0.025 & 0.1 \\
      Sequence Length &  128 & 128 & 512 & 512 \\
      Batch Size &  4096 & 32 & 4096 & 32 \\
      Checkpoint Interval & 20k & 3.5k & 20k & 3.5k \\
      \bottomrule
    \end{tabular}
    \end{center}
  \vspace{-0.3cm}
    \caption{Hyperparameters for all models}
    \label{tab:hyperparameters}
\end{table}

\paragraph{Translate} For our Translate experiments we follow the setup of \citet{vaswani2017attention} and use their Tensor2Tensor framework \citep{vaswani2018tensor2tensor}. We train using WMT18 ((Europarl v7, Common Crawl corpus and News Commentary v13) en-de, de-en, en-cs and cs-en datasets. We report BLUE scores provided by SacreBLEU \citep{Post2018ACF} on newstest 2018 dataset. We train a 6 layer Transformer model. Any changes to position encoding are applied to all the attention layers both in the encoder and decoder. We use Adam optimizer and train for 250k steps. For decoding we use beam search with beam size 10 and length penalty 0.6.

\section{Proofs} \label{appendix:proofs}
\begin{proof}[Proof of Theorem~\ref{thm:rank}]
The first claim follows easily by observing that rank of product of an two matrices is upper bounded by the minimum of the individual ranks.
\begin{align*}
    rank(\rmA_a) &= rank((\rmX+\rmP)  \rmW_Q  \rmW_K^\top (\rmX+\rmP)^\top) \\
    &\le \min(rank(\rmX+\rmP), rank(\rmW_Q), rank(\rmX+\rmP), rank(\rmW_K)) \\
    &\le d_h.
\end{align*} 
\begin{align*}
    rank((\rmX+\rmP)  \rmW_Q  \rmW_K^\top (\rmX+\rmP)^\top) \le d_h, \text{where } \rmW_Q, \rmW_K \in \R^{d \times d_h}
\end{align*} 

The last inequality follows from $rank(\rmW_Q) \leq d_h$ as $\rmW_Q \in \R^{d \times d_h}$.

To prove the second claim we follow a construction approach. Let us first take $\rmW_Q = \rmW_K$ to be same matrices with first $d_h$ rows being identity matrix and the remaining $d-d_h$ rows being all zeros. Then $$\rmW_Q \rmW_K^\top = {
                \left( 
                      \begin{array}{cc}
                             I_{d_h, d_h} & 0_{d_h, d-d_h}\\
                             0_{d-d_h, d_h} & 0_{d-d_h, d-d_h}
                      \end{array}
                \right).
              }$$
Here $I_{d_h, d_h}$ denotes the identity matrix in $\R^{d_h \times d_h}$ and $0_{d_h, d}$ denotes the all zeros matrix in $\R^{d_h, d}$. 

We let $\rmX$ be such that the first $d$ rows form an identity matrix and rest are zeros - $\rmX^\top = [I_{d, d} , 0_{n-d, d}]$.
Hence $\rmX \rmW_Q \rmW_K^\top  X^\top $ becomes a similar diagonal matrix with $$ \rmX \rmW_Q \rmW_K^\top  \rmX^\top  = {
                \left( 
                      \begin{array}{cc}
                             I_{d_h, d_h} & 0_{d_h, n-d_h}\\
                             0_{n-d_h, d_h} & 0_{n-d_h, n-d_h}
                      \end{array}
                \right).
              }$$

Choose $d_p = n > d_h$ and let $\hat{\rmP} = I$. Now chosing $\hat{\rmP}$ with zeros in the first $n-d_p$ columns and identity in the last $d_p$ columns ($\hat{\rmP} = [0_{d,n-d_p}, I_{d_p, d_p}]$) gives $$ \hat{\rmP} \hat{\rmP}^\top = {
                \left( 
                      \begin{array}{cc}
                             0_{n-d_p, n-d_p} & 0_{n-d_p, d_p}\\
                             0_{d_p, n-d_p} & I_{d_p, d_p}
                      \end{array}
                \right).
              }$$
Combining these two gives us \begin{align*}
    rank(\rmA_r) &= rank(\rmX \rmW_Q \rmW_K^\top \rmX^\top + \hat{\rmP} \hat{\rmP}^\top) \\
    & = min(d_h + d_p, n)  > d_h.
\end{align*}
\end{proof}

Let $\rmX \in \R^{n \times d}$ be the input word embeddings in dimension $d$ with sequence length $n$. We have trainable position embeddings $\rmP \in \R^{n \times d}$, which are added to the input sequence before feeding into the model $g$. For a given input $\rmX$ and label $y$, the objective for a loss function $\ell$ is as follows: \begin{align} \label{eq:objective}
    L = \ell\left( g(\rmX + \rmP), y \right)
\end{align}

\begin{theorem}\label{thm:gradients}
Let $\rmX$ and $\rmP$ be trainable embedding matrices in $\R^{n \times d}$. Then the gradients of the loss function in equation \eqref{eq:objective}, at any point $(\rmX, y)$, and for any differentiable functions $\ell$ and $g$, are same for $\rmX$ and $\rmP$.
\end{theorem}

\noindent \paragraph{Remarks.} This theorem shows us that the gradients are same for the input token embeddings and position embeddings. While in standard NLP tasks the inputs $X$ can be different in each step due to different input tokens being present in each mini batch, the result still suggests that additive position embedding can limit the model from learning the relative importance of position encodings with respect to token embeddings based on the training task at hand. 

\begin{proof}[Proof of Theorem~\ref{thm:gradients}]
The above theorem follows by just computing the gradients and showing they are equal for each step.

Gradients of the above objective w.r.t $\rmX$ and $\rmP$ are as follows. \begin{align*}
    &\nabla_{\rmX} L = \nabla_{g} L \cdot \nabla_{\rmX + \rmP} g \cdot \nabla_{\rmX} (\rmX + \rmP) \\
    & = \nabla_{g} L \cdot \nabla_{\rmX + \rmP} g \\
    &\nabla_{\rmP} L = \nabla_{g} L \cdot \nabla_{\rmX + \rmP} g \cdot \nabla_{\rmP} (\rmX + \rmP) \\
    & = \nabla_{g} L \cdot \nabla_{\rmX + \rmP} g.
\end{align*}
The above computation of gradient follows from chain rule. This shows that the gradients of $L$ w.r.t. $\rmX$ and $\rmP$ are the same.
\end{proof}

\pagebreak

\section{Attention Visualization} \label{appendix:visualization}
In this section, we examine the model internals to understand how the proposed model works. We first visualize the model internals of different modeling alternatives to argue our proposed model is sensible.

\paragraph{Why We Remove the Input Embedding}
To understand if it is sensible to remove the input additive embedding after adding position scalars per-head, we add additive position embedding to our \absscalar{} model. Then, we examine the position embedding of the BERT model and our \absscalar{} variant with additive position embedding. 
Figure~\ref{fig:position_embedding_inner_prod} shows that, when the model has both absolute scalar and  additive absolute position embedding, the position embedding encodes almost no information --- all position embeddings at input are similar.

\begin{figure}[H]
\centering
  \includegraphics[scale=0.4]{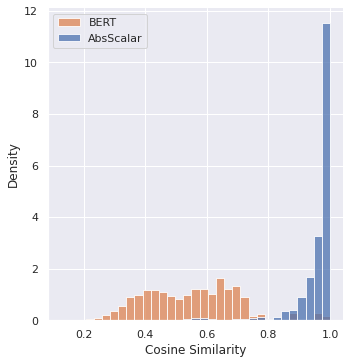}
\caption{The cosine similarity distribution between all absolute position pairs of the input additive positional embedding for the baseline BERT model and the proposed \absscalar{}. We observed that, after the position features are added to each head as in \absscalar{}, the input position embedding contains almost no information  --- all input position pairs are similar.}
\label{fig:position_embedding_inner_prod}
\end{figure}

\paragraph{The Effect of Segment Attention}
We also examine the effect of adding segment attention on top of the position attention. Figure~\ref{fig:pos_seg_visualization} shows some representative patterns. We observe that segment attention enables the model to attend more to parts of the sequence that belongs to certain segments.

\begin{figure}[H]
\centering
\begin{subfigure}{.5\textwidth}
  \centering
\includegraphics[width=6.5cm]{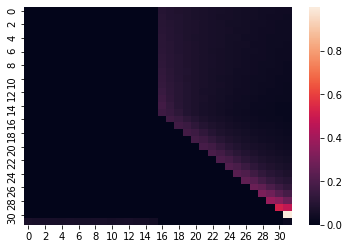}
\caption{Attend to the Second Segment}
\label{fig:shaw_att_scales_v2}
\end{subfigure}%
\begin{subfigure}{.5\textwidth}
  \centering
\includegraphics[width=6.5cm]{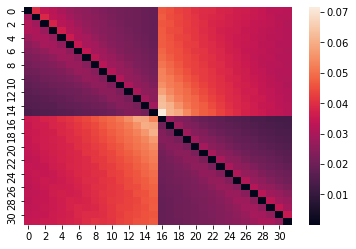}
\caption{Down-weight Relative Position Attention}
\end{subfigure}
\caption{We consider input of length 32 with two segments. The second segment starts at index 16. We observe the attention patterns in the \ours{} model without token-to-token attention.}
\label{fig:pos_seg_visualization}
\end{figure}

\pagebreak

\paragraph{Shifting Pattern Learned from Absolute Positional Attention}
Using relative position encoding gives generally better results despite smaller improvement scale compared to moving feature encoding per-head. To understand this, we visualize the attention pattern of the absolute positional attention and found two representative patterns in \dietabs{} in Figure~\ref{fig:abs_scalar_attention_viz}. We observe that even given absolute position features, the model learns a ``shifting pattern'' for the most part. Different from \citet{wang2020position} which claimed absolute position only learns local patterns, we show the position attention can actually attend to longer context. However, the shifting pattern can be modeled directly by relative position. Thus, \dietrel{} can be a better model choice with fewer parameters and more accurate inductive bias in some applications. 

\begin{figure}[H]
\centering
\begin{subfigure}{.5\textwidth}
  \centering
\includegraphics[width=6.5cm]{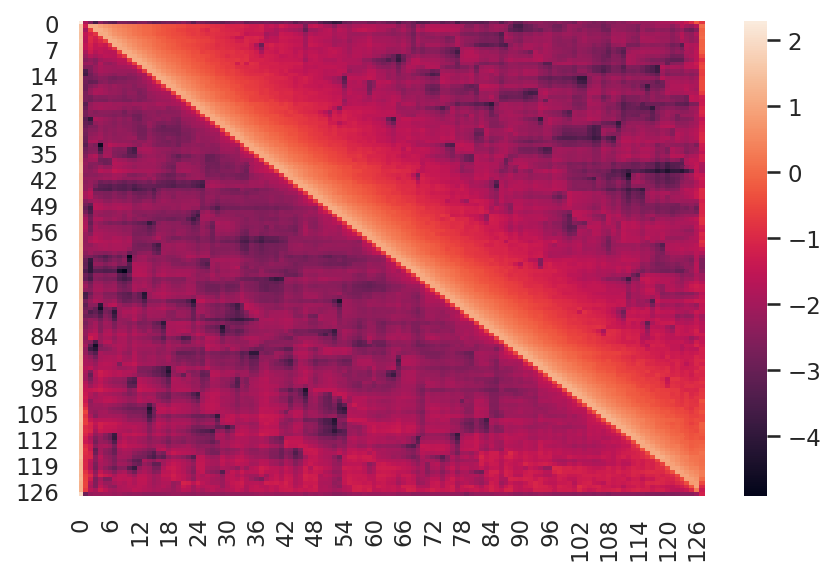}
\caption{Attend to Forward Neighbors}
\label{fig:shaw_att_scales_v3}
\end{subfigure}%
\begin{subfigure}{.5\textwidth}
  \centering
\includegraphics[width=6.5cm]{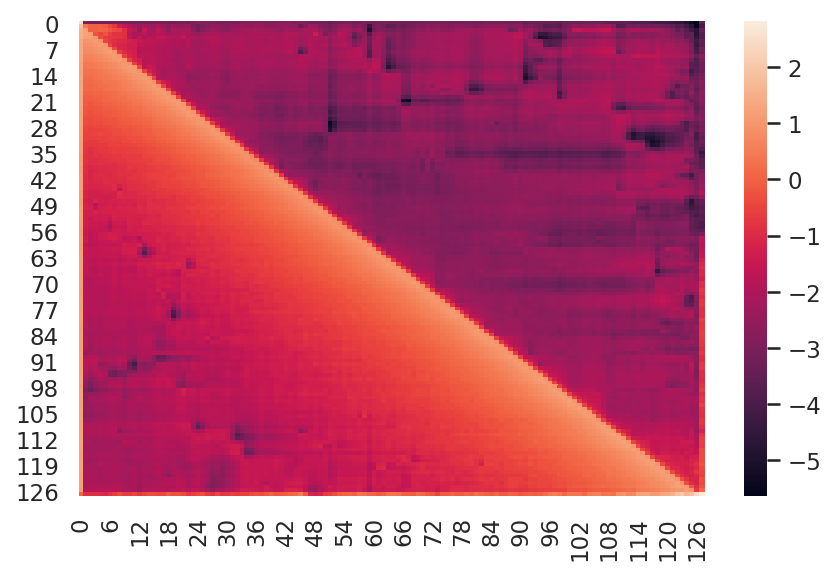}
\caption{Attend to Previous Tokens}
\end{subfigure}

\caption{Sampled position attention score patterns for the \absscalar{} model. We can see a clear shifting patterns generated by the model. Such patterns can be modeled better by relative positional scalar encodigs.}
\label{fig:abs_scalar_attention_viz}
\vspace{-0.1in}
\end{figure}

\paragraph{Rank of Positional Attention Matrices}
In Figure~\ref{fig:pos-att-rank}, we present a comparison of rank of position attention matrices for a $\BB$ model with absolute position embeddings at input ($\rmP_Q \rmW_Q \rmW_K^\top \rmP_K^\top$) v.s. absolute position embeddings per-head (\absscalar{}~\eqref{eq:absscalar}, ($\rmP_Q \rmP_K^\top$),  where $\rmP_Q, \rmP_K \in \R^{n \times d_p})$. With additive positional embedding at input, position attention matrices have much lower rank, limiting the representative power. This is alleviated by \dietabs{}.
\begin{figure}[H]
    \centering
    \includegraphics[width=7.5cm]{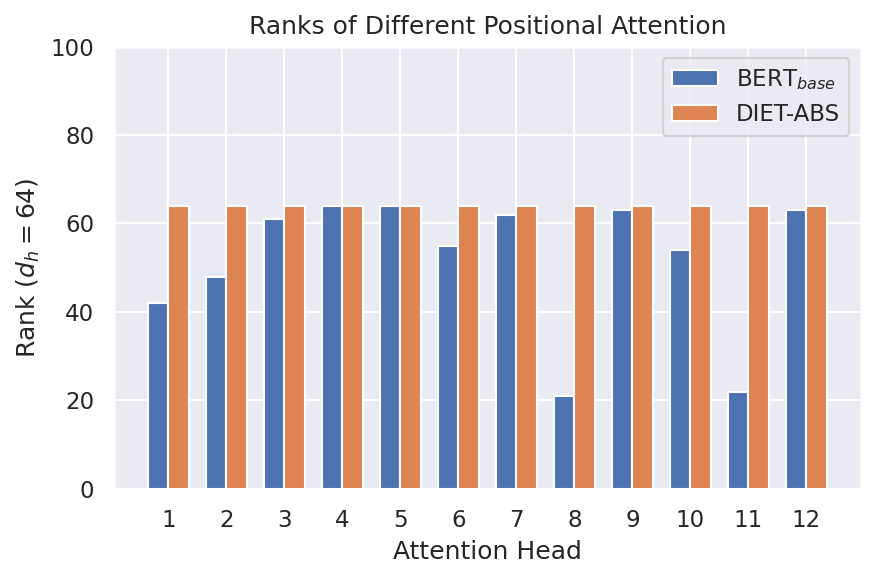}
    \caption{Rank of positional attention matrices}
    \label{fig:pos-att-rank}
\end{figure}

\pagebreak

\section{Additional Ablation Study on GLUE} \label{appendix:glue}
Earlier we present an ablation study on XTREME in Table~\ref{tab:xtreme-scalar-sharing} for decoupled positional attention variants. We compare \dietrel{} and \dietabs{} against the baseline~\citep{devlin2018bert}. We now present a similar study on the GLUE benchmark in Table~\ref{tab:glue_ablation_model} and observe similar results. 

\paragraph{Positional Encoding}
In Table~\ref{tab:glue_ablation_model}, moving positional embeddings from input to per-head improves average score for both \dietrel{} (+0.1\%) and \dietabs{} (+0.2\%).

\paragraph{Segment Encoding}
In Table~\ref{tab:glue_ablation_model}, moving segment embeddings from input to per-head improves both \dietrel{} (+0.3\%) and \dietabs{} (+0.05\%).

\paragraph{Sharing Strategies}
Sharing plays an important role for \dietabs{}. In Table~\ref{tab:glue_ablation_sharing}, we find that sharing will degrade the performance of \dietrel{} (-0.2\% layer-wise, -0.3\% head-wise). For \dietabs{}, sharing makes the model more stable, and able to compete with \dietrel{}.

\begin{table}[H]
\small
    \begin{center}
        \begin{tabular}{l | c | c | cccccc | c}
            \toprule
            
             \multirow{2}{*}{Model} &
             \multirow{2}{*}{Position} &
             \multirow{2}{*}{Segment} &
             \bf MNLI & \bf QQP & \bf QNLI & \bf SST2 & \bf CoLA &\bf STS-B & \bf \multirow{2}{*}{Avg} \\
             &  & & 393k & 364k & 105k & 67k & 8.5k & 7k \\
            \midrule
            \citet{devlin2018bert} & input & input & 85.8 / 85.9 & 91.1 & 89.9 & 93.2 & 58.7 & 89.0 & 84.8 \\
            \dietrel  &  per-head & input & 86.0 / 86.1 & 91.0 & 89.8  & 92.8  & 59.6 & 89.0 & 84.9 \\
            \dietrel &  per-head & per-head & 86.3 / 86.3 & 91.0 & 90.5 & 92.9 & 60.3& 89.3 & 85.2 \\
            \dietabs{} ($d_p$=64) &  per-head & input & 86.1 / 85.8 & 91.2 & 90.0 &  93.0 & 58.9 & 89.9 & 85.0 \\
            \dietabs{} ($d_p$=64) &  per-head & per-head & 86.1 / 86.1 & 91.2 & 90.2 & 93.0 & 58.9 & 89.8 & 85.0 \\
            \dietabs{} ($d_p$=64, share) & per-head & per-head & 86 / 86.8 & 91.1 & 90.4 & 92.9 & 59.3 & 89.8 &  85.2 \\
            \dietabs{} ($d_p$=128, share) & per-head & per-head & 86.4 / 86.4 & 90.8	& 89.5	& 93.0	& 59.8	& 90.2 & 85.2 \\ 
            \bottomrule
        \end{tabular}%
        \caption{
        Position and segment ablation study on GLUE: \dietrel{} and \dietabs{} demonstrate the advantages of moving both positional and segment embedding from input to per-head.
        }
    \label{tab:glue_ablation_model}
    \end{center}
\end{table}

\begin{table}[H]
\small
    \begin{center}
        \begin{tabular}{l | c |  cccccc | c}
            \toprule
             \multirow{2}{*}{Model} &
             \multirow{2}{*}{Sharing} &
             \bf MNLI & \bf QQP & \bf QNLI & \bf SST2 & \bf CoLA &\bf STS-B & \bf \multirow{2}{*}{Avg} \\
             & & 393k & 364k & 105k & 67k & 8.5k & 7k \\
            \midrule
            \dietrel{} & - & 86.3 / 86.3 & 91.0 & 90.5 & 92.9 & 60.3& 89.3 & 85.2 \\
            \dietrel{} & layer-wise &  86.5 / 86.3 & 91.1 & 90.0 & 93.0 & 58.8 & 89.6 & 85.0\\
            \dietrel{} & head-wise & 85.8 /  85.7& 91.2 & 90.2 & 92.8 & 59.8 & 89.1 & 84.9\\
            \dietabs{} ($d_p$=64) & - & 86.1 / 86.1 & 91.2 & 90.2 & 93.0 & 58.9 & 89.8 & 85.0 \\
            \dietabs{} ($d_p$=128) & - & 86.7 / 86.5  &  91.2 &  90.6 & 92.8  & 60.1 & 89.4 & 85.3 \\
            \dietabs{} ($d_p$=64) & layer-wise & 86 / 86.8 & 91.1 & 90.4 & 92.9 & 59.3 & 89.8 &  85.2 \\
            \dietabs{} ($d_p$=128) & layer-wise & 86.4 / 86.4 & 90.8	& 89.5	& 93.0	& 59.8	& 90.2 & 85.2 \\ 
            \bottomrule
        \end{tabular}%
        \caption{Sharing ablation study on GLUE: We run ablation study to understand the effect of sharing position encoding parameters across layers and heads. We notice that sharing improves the performance of \dietabs{}, but hurts the performance of \dietrel{} with both layer-wise or head-wise sharing. 
        }
    \label{tab:glue_ablation_sharing}
    \end{center}
\end{table}